\documentclass[runningheads]{llncs}
\usepackage[T1]{fontenc}
\usepackage{graphicx}
\usepackage{amsmath}
\usepackage{amssymb}

\usepackage{amsthm}
\usepackage[noadjust]{cite}
\usepackage[ruled,vlined,linesnumbered]{algorithm2e}
\usepackage{xcolor}
\usepackage{multirow}
\usepackage{subcaption}
\usepackage{bm}
\usepackage{comment}
\usepackage[font=small,labelfont=bf]{caption}
\usepackage{url}
\usepackage{array}

\newcommand{\Exp}{\mathbb{E}}
\newcommand{\R}{\mathbb{R}}
\newcommand{\mB}{\mathcal{B}}
\newcommand{\mC}{\mathcal{C}}
\newcommand{\mK}{\mathcal{K}}
\newcommand{\mN}{\mathcal{N}}
\newcommand{\mM}{\mathcal{M}}

\newcommand{\tuple}[1]{\langle #1 \rangle}
\newcommand{\norm}[1]{\Vert #1 \Vert}
\newcommand{\abs}[1]{\vert #1 \vert}
\newcommand{\red}[1]{\textcolor{red}{#1}}

\SetKw{KwOr}{or}
\SetKw{KwBreak}{break}
\SetKw{KwTrue}{true}
\SetKwInOut{KwInput}{Input}
\SetKwInOut{KwOutput}{Output}

\begin{document}
\title{Incentive-Compatible and Distributed Allocation for Robotic Service Provision Through Contract Theory
\thanks{This work has been submitted to the IROS 2024 for review.}
}
\titlerunning{Contract-Theoretic and Distributed Robot Allocation}
%
\author{Yuhan Zhao\inst{1} \and
Quanyan Zhu\inst{1}}
\authorrunning{Y. Zhao and Q. Zhu}
\institute{New York University, Brooklyn NY 11201. \\
\email{\{yhzhao, qz494\}@nyu.edu}}
\maketitle
\begin{abstract}
Robot allocation plays an essential role in facilitating robotic service provision across various domains. Yet the increasing number of users and the uncertainties regarding the users' true service requirements have posed challenges for the service provider in effectively allocating service robots to users to meet their needs.
In this work, we first propose a contract-based approach to enable incentive-compatible service selection so that the service provider can effectively reduce the user's service uncertainties for precise service provision. Then, we develop a distributed allocation algorithm that incorporates robot dynamics and collision avoidance to allocate service robots and address scalability concerns associated with increasing numbers of service robots and users.
We conduct simulations in eight scenarios to validate our approach. Comparative analysis against the robust allocation paradigm and two alternative uncertainty reduction strategies demonstrates that our approach achieves better allocation efficiency and accuracy\footnote{The simulation codes are available at \url{https://github.com/yuhan16/Contract-Allocation}.}.

\end{abstract}
\section{Introduction} \label{sec:intro}

Rapid technological advances have made robotic services prevalent in various fields, such as social robotics \cite{arduengo2021robot,Prassler2016}, manufacturing \cite{afrin2021resource}, and transportation \cite{kim2020autonomous,liu2019dynamic}. In robotic service provision, a service provider (SP) controls a group of versatile service robots and assigns them to fulfill users' service requests like status checks and device maintenance.
A typical example is on-demand cargo delivery in smart factories, where multiple robots are assigned to different areas to take care of the order transportation in the corresponding area. The fundamental technique for robotic service provision is robot allocation. The SP needs to allocate robots to different places or users to complete various service requests and achieve higher autonomy. Similar applications such as robot patrolling in urban areas \cite{huang2019survey,basilico2022recent}, mobile robot-based communication networks \cite{andre2014application,zeng2016wireless}, and robot inspection and maintenance in the infrastructure \cite{franko2020design} all reveal the importance of effective and efficient robot allocation.

However, emerging changes and challenges require additional considerations for designing effective robot allocation methods. First, the SP faces more nuanced services as the user service requirements become more detailed and the service robots more advanced. Unlike different exclusive services, these services are inclusive and vary across different levels. For example, a high-quality service (e.g., inspection and repair) robot can complete a low-level service (e.g., inspection only). Therefore, the SP needs to accurately allocate service robots to maximize resource utilization.
Second, uncertainties for the user can lead to inefficient robot allocation and resource mismatches. The uncertainty can arise from the user's selfish incentives to misreport service requests. For example, a user can receive more benefits (e.g., faster service completion) if he requests a higher level of service from the SP. Third, the SP can encounter scalability issues as the number of service robots and users increases. Besides, practical allocation solutions should also consider robot dynamics and collision avoidance.
Most classic methods in resource or task allocation are centralized and require full user information to make allocation plans. They typically produce direct user-robot assignment plans without detailed robot-level trajectories for allocation. Besides, they primarily focus on different task scenarios rather than different levels of services. Therefore, there is a need for distributed algorithms to address these concerns and provide robot-level allocation plans.

\begin{figure}[t]
    \centering
    \includegraphics[height=5cm]{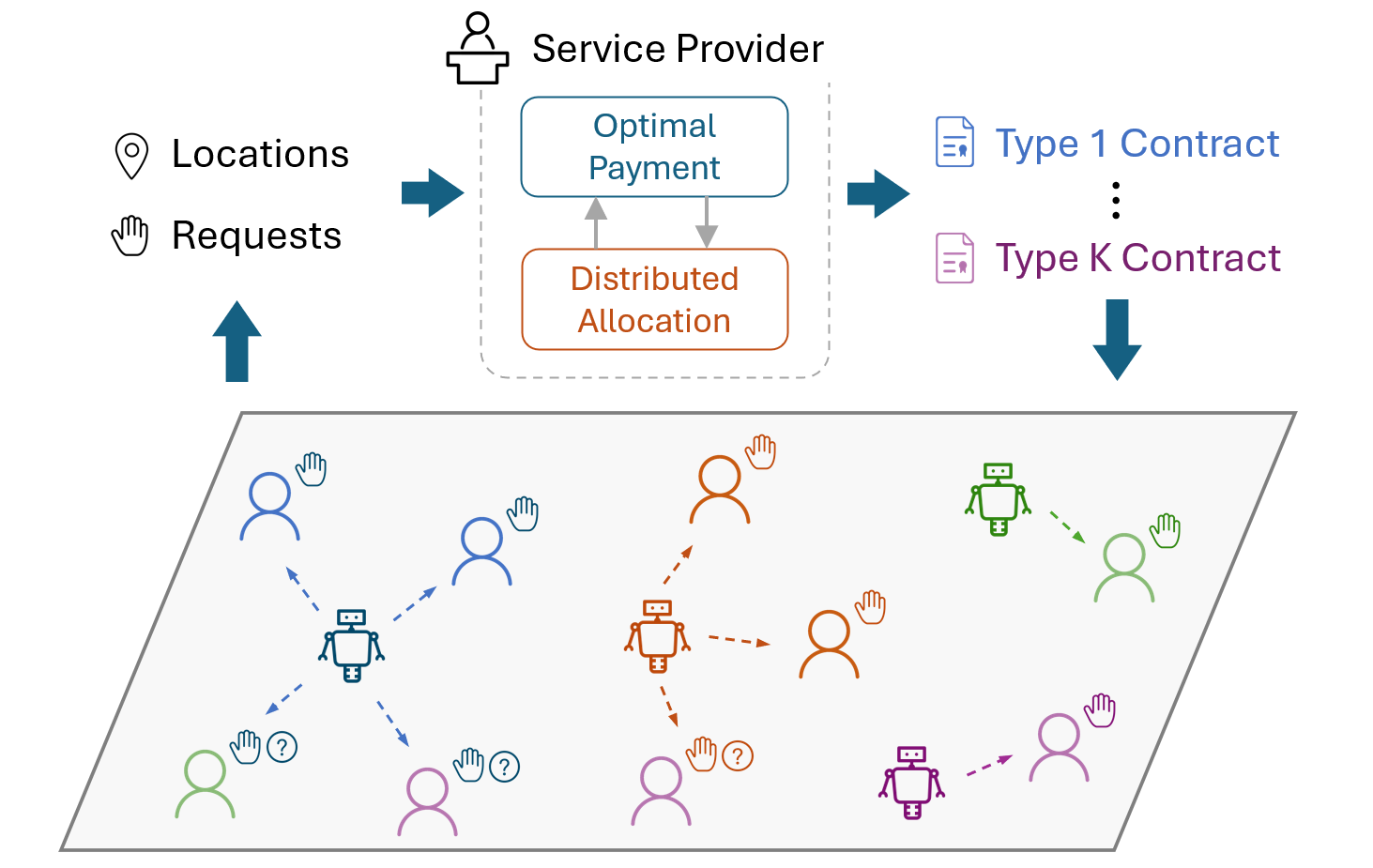}
    \caption{The service provider assigns various service robots to users based on their specific service requests. However, users may inaccurately report their needs, resulting in resource mismatches. We propose a contract-based approach, including optimal payment design and distributed allocation, to ensure accurate and efficient allocation of service robots to designated users.
    }
    \label{fig:intro}
\end{figure}

To this end, we propose a contract-based approach to reduce the SP's uncertainty regarding user needs and achieve distributed and efficient robot allocation for service provision, as demonstrated in Fig.~\ref{fig:intro}.
Rooted in game theory, contract theory studies the incentive-compatible mechanisms to allocate resources and mitigate the \emph{moral hazard} in scenarios of asymmetric information structure \cite{bolton2004contract,borgers2015introduction}. The contract designer designs a menu of contracts to incentivize the participants to truthfully select options that align with their needs, considering that the participants have alternative choices. It provides a suitable framework to characterize multi-level service misuses in our problem.
Specifically, the SP, lacking knowledge of a user's true service requirement, offers a range of incentive-compatible contracts. Each contract specifies a distinct level of service, and the user selects the contract that matches his true needs. Leveraging contract theory, the SP gains comprehensive user information and reduces the user's uncertainty.
Subsequently, we develop a distributed allocation algorithm to accurately assign suitable service robots to users while ensuring collision avoidance. 
We use simulations in eight testing scenarios to corroborate our approach. Results show that our approach successfully completes all the test cases and yields collision-free robot trajectories to achieve effective allocation. We also compare the robust allocation scheme and two alternative user uncertainty reduction methods and demonstrate that our approach consistently outperforms them in allocation efficiency and accuracy.

\emph{Notations:} We denote $\mN := \{1,2,\dots,N\}$, $\mM := \{1,\dots, M\}$, and $\mK:= \{1,\dots, K\}$. The Euclidean distance is denoted by $\norm{\cdot}$.

\section{Related Work} \label{sec:related_work}
Robot allocation is related to multi-robot task allocation (MRTA), which studies task assignment to different robots to achieve the overall task objective. Many works have thoroughly reviewed the advances and related applications in MRTA \cite{gerkey2004formal,dias2006market,khamis2015multi}. The methodologies in MRTA can be generally categorized into optimization-based and market-based approaches. Optimization-based approaches solve an optimal assignment problem to maximize some overall allocation utility. The assignment problems are general mixed integer programming (MIP), and the results imply which robot to assign to which task. Examples include \cite{notomista2021resilient,mahulea2017robot,mayya2021resilient}. 
Market-based approaches treat robots as individual agents to conduct task allocations. The auction is the most common mechanism in the market-based approaches, where every robot submits a bid for a task to the auctioneer (the planner) and the winner receives the task. Depending on the auction rules, the auction can be further divided into the first-price auction and the second-price auction, etc. Concrete applications can be found in package transportation \cite{bai2022group}, multi-robot routing \cite{bai2021distributed}, and cooperative tracking \cite{capitan2013decentralized}. 
However, most optimization-based methods are centralized and require full information on both tasks and robots to perform allocation. They also face challenges when solving large-scale MIP. The market-based methods, despite the game-theoretic formulation, require extra knowledge between bidders\footnote{The bidders are the users in our problem and they need service robots to fulfill the service need.} to make the bid for every task. Besides, due to the user's selfish incentives, users are all likely to ask for the highest quality of the service, making the allocation ineffective. Our approach aims to bridge the gap and provide an efficient allocation method to address the challenges.

When there are uncertainties in the allocation problem, robust task allocation provides a feasible solution. It analytically characterizes the unknown part of the planning model and provides robustness in allocation in the face of uncertainty. For example, Choi et al. in \cite{choi2009consensus} have investigated task allocation for multi-vehicle coordination under inconsistent situational awareness across the vehicles and variations
in the communication network topology.
A robust robot assignment method developed by Prorok in \cite{prorok2020robust} has addressed the problem of assigning mobile robots over a transportation network with uncertain travel time and minimizing the average waiting time at the destination.
Choudhury et al. in \cite{choudhury2022dynamic} have proposed a dynamic allocation algorithm that minimizes the number of unsuccessful tasks during the operation horizon under the task completion uncertainty.
Despite the extra robustness, robust task allocation sometimes sacrifices planning efficiency. Our method leverages the contract-based approach to reduce user uncertainty in service provision and, hence, achieve more efficient allocation than its robust counterparts. 

Contract design has rich literature in operations research \cite{zhao2010coordination,qin2020contract}, as well as engineering like IoT \cite{chen2023qos} and smart cities \cite{perera2014sensing}. A more comprehensive introduction to contract theory can be found in \cite{bolton2004contract}.
There are works that leverage contract theory for resource allocation although not for robotic services. For example, Zhou et al. in \cite{zhou2019computation} have adopted contract theory to offload the computation tasks of a vehicular network from the base station for efficient communication. Lim et al. in \cite{lim2021towards} have developed a contract-based matching algorithm to dispatch drones to different users to fulfill the user's computational demands. Chang et al. in \cite{chang2020incentive} has developed a contract-theoretic network resource allocation mechanism between a network operator and infrastructure providers to improve the network operator's payoff.
In this work, we focus on robotic allocation and distributed implementation based on contract theory.

\section{Problem Formulation} \label{sec:prob}

\subsection{Basic Settings} \label{sec:prob.setting}
We consider a service provider (SP, she) to have a group of $N$ robots that provide $K$ types of service. Each robot has a type $k$ and provides the type $k$ service to the user, $k \in \mK$. A higher type can represent a better service quality or more efficient service provision. For example, a type 2 robot can fix the user's problem faster than a type 1 robot. We denote the number of type $k$ robots by $N^k$, $k \in \mK$, and $\sum_{k \in \mK} N^k = N$. 
Here, we assume that a type $k$ robot can fulfill all the services below type $k$. For example, a type $2$ robot can provide type $1$ and type $2$ services to the user while a type $1$ robot can only complete type $1$ service.

We consider $M$ users in a working space $\mathcal{W}\subset \R^2$ with the position $q_i \in \R^2$, $i \in \mM$. Every user has a type $\theta_i \in \mK$ service to fulfill. However, due to selfish incentives, the user can request a higher type $\phi_i$ service to gain extra benefits. For example, a user with a type 1 service need can request a type 2 robot for more efficient service provision.
We use $\tuple{\theta_i, \phi_i}$ to denote the user $i$'s true service type demand and the requested type. The SP does not know the user's true service type except for a type distribution $p_i \in \Delta(K), i \in \mM$. She can only assign the robot to the user based on the requested type $\phi_i$, which may cause resource waste. Therefore, we leverage contract theory to enable the SP to provide users with proper service robots that are compatible with their true service types.

A contract contains an allocation rule and a payment rule. The allocation rule determines which service robot to assign to the user and where to place the robot; the payment rule specifies the service price for the user to use the robots.
We capture the allocation rule by $\tuple{b^k_{ij}, x^k_{j}}, i \in \mM, j \in \mN^k, k \in \mK$, where $b^k_{ij} = \{0,1\}$ indicates if the type $k$ service robot $j$ robot is assigned to the user $i$; $x^k_{j} \in \R^2$ is the robot's location. The payment rule $\rho^k_{ij}$ is the service price for the user $i$ to use the type $k$ service robot $j$. 
For simplicity, we denote $\bm{b}^k := \{\{b^k_{ij}\}_{j=1}^{N^k}\}_{i=1}^M$, $\bm{x}^k := \{x^k_{j}\}_{j=1}^{N^k}$, and $\bm{\rho}^k := \{\{\rho^k_{ij}\}_{j=1}^{N^k}\}_{i=1}^M$. We also write $\bm{b} := \{\bm{b}^k\}_{k=1}^K$, $\bm{x} := \{\bm{x}^k\}_{k=1}^K$, and $\bm{\rho} := \{\bm{\rho}^k\}_{k=1}^K$. Therefore, the SP's contract can be denoted by $\mC := \tuple{\bm{b}, \bm{x}, \bm{\rho}}$. 

\begin{remark}
A service robot can be assigned to multiple users due to the large number of users. Then, the robot needs to consider the order of service provision if it only serves one user at a time. However, the ordering problem can be regarded as independent of contract design. It is sufficient for the SP to determine robot positions $\bm{x}$ for fast service provision in the contract. Once $\bm{x}$ is found, we can leverage other methodologies, such as traveling salesman problems (TSP) to address the order of service provision.
\end{remark}

\subsection{User Model}
Users request their service demand based on the contract $\mC$. Specifically, if user $i$ requests a type $\phi_i$ service to the SP, a service robot $j$ with type $\phi_i$ located at $x^{\phi_i}_{j}$ will be assigned to the user $i$ with the price $\rho^{\phi_i}_{ij}$. We define the user $i$'s utility function by 
\begin{equation*}
    u_i(\phi_i) = b^{\phi_i}_{ij}\left( g(\phi_i - \theta_i) r - \rho^{\phi_i}_{ij} \right).
\end{equation*}
Here, $\theta_i$ is the user's true service type; $r$ is a parameter that measures the user $i$'s gain of receiving the service. However, a user can receive more gains by misreporting his true type. We capture this extra benefit by the function $g: \R\to\R$ with $g(x)$ increasing and concave in $x \geq 0$ and $g(0) = 1$. For example, user $i$ with type 1 receives a service gain $r$ if he truthfully reports $\phi_i = 1$; he can receive a higher gain $g(1) r > r$ if he misreports $\phi_i = 2$. We also define that $g(x) = 0$ for $x < 0$ since a high-type service cannot be completed by a low-type robot, resulting in a zero gain.

\subsection{Service Provider Model}
The SP designs the contract $\mC$ to incentivize users to truthfully report their service types and allocate service robots. In the allocation, each user is assigned one robot with a specific type, which is captured by the following allocation constraints
\begin{equation}
\label{eq:alloc}
    \sum_{j=1}^{N^k} b^k_{ij} = 1, \ b^k_{ij} = \{0,1\}, \quad \forall i \in \mM, \ k \in \mK.
\end{equation}
In addition, the SP needs Incentive Compatibility (IC) constraints and Individual Rationality (IR) constraints to ensure that the resources are not misused.

IR constraints ensure that all users have an incentive to take the contract, meaning that users should have a positive utility when requesting the true service using the contract $\mC$:
\begin{equation}
\label{eq:ir}
    b^k_{ij} \left( r - \rho^k_{ij} \right) \geq 0, \quad \forall i \in \mM, \ j \in \mN^k, \ k \in \mK.
\end{equation}
The constraints hold trivially when $b^k_{ij} = 0$.

IC constraints ensure that users truthfully report their service types, preventing them from misusing service resources to get extra benefits. IC states that a user's utility is maximized if and only if he requests a service type that matches his true type. With the allocation constraints \eqref{eq:alloc}, we write the IC constraints as two parts:
\begin{equation}
\label{eq:ic1}
\begin{split}
    \sum_{j=1}^{N^k} b^k_{ij} \left( r-\rho^k_{ij} \right) \geq& \sum_{j=1}^{N^l} b^l_{ij} \left( 0- \rho^l_{ij} \right), \\ 
    \forall l = 1,&\dots, k-1, \ k = 2,\dots, K, \ i \in \mM. 
\end{split}
\end{equation}
\begin{equation}
\label{eq:ic2}
\begin{split}
     \sum_{j=1}^{N^k} b^k_{ij} \left( r-\rho^k_{ij} \right) &\geq \sum_{j=1}^{N^l} b^l_{ij} \left( g(l-k) r - \rho^l_{ij} \right), \\ 
     \forall l = k+1, &\dots, K, \ k = 1,\dots, K-1, \ i \in \mM. 
\end{split}
\end{equation}
The first part \eqref{eq:ic1} (and the second part \eqref{eq:ic2}) captures the cases where the user selects a lower (higher) service type than his true one.

Also, the payment in $\mC$ should be non-negative so that the SP can make a profit:
\begin{equation}
\label{eq:payment}
    \rho^k_{ij} \geq 0, \quad \forall i \in \mM, \ j \in \mN^k, \ k \in \mK. 
\end{equation}

The SP's objective of designing a contract is twofold: maximizing the net revenue and the allocation efficiency. The revenue comes from the payment, which is captured by
\begin{equation}
\label{eq:revenue}
    \Exp_p [R(\bm{b}, \bm{\rho})] = \sum_{i=1}^M \sum_{k=1}^K \sum_{j=1}^{N^k} p^k_i b^k_{ij} \rho^k_{ij}.
\end{equation}
The allocation efficiency is measured by the locational energy, which reflects the quality of robot deployment. We define the expected locational energy as  
\begin{equation}
\label{eq:location_energy}
    \Exp_p [L(\bm{x},\bm{b})] = \sum_{i=1}^M \sum_{k=1}^K \sum_{j=1}^{N^k} p^k_i b^k_{ij} f \left( \norm{q_i - x^k_{j}} \right),
\end{equation}
where $f:\R \to \R$ is some positive and increasing function with $f(0) = f'(0) = 0$. A smaller $L(\bm{x},\bm{b})$ implies a better allocation efficiency,
which can correspond to a shorter routing time for robots to visit all appointed users from the positions $\bm{x}$.
Let $\gamma > 0$ be the weighting parameter. The SP's contract design problem is formulated as follows:
\begin{equation}
\label{eq:contract}
\begin{split}
    \max_{\bm{x}, \bm{b},\bm{\rho}} \quad & \Exp_p [R(\bm{b}, \bm{\rho})] - \gamma \Exp_p [L(\bm{x},\bm{b})] \\ 
    \text{s.t.} \quad & \eqref{eq:alloc}-\eqref{eq:payment}.
\end{split}
\end{equation}

\section{Contract Analysis and Distributed Implementation} \label{sec:analysis}
The contract problem \eqref{eq:contract} is a mixed integer programming (MIP) and is challenging to solve directly, especially when $f$ is nonlinear. To develop efficient algorithms to obtain the contract, we split \eqref{eq:contract} into two sub-problems and analyze them sequentially.

\subsection{Payment Sub-Problem}
The payment sub-problem is obtained by fixing the allocation variables $\tuple{\bm{b}, \bm{x}}$, which gives
\begin{equation}
\label{eq:contract.pay}
\begin{split}
    \max_{\bm{\rho}} \quad & \sum_{i=1}^M \sum_{k=1}^K \sum_{j=1}^{N^k} p^k_i b^k_{ij} \rho^k_{ij} \\ 
    \text{s.t.} \quad & \eqref{eq:ir}-\eqref{eq:payment}.
\end{split}
\end{equation}
The following proposition characterizes the optimal solution of the payment problem \eqref{eq:contract.pay} given the allocation variables.

\begin{proposition} \label{prop:1}
Assume $g(1) < \frac{K}{K-1}$. Given the allocation variable $\tuple{\bm{b}, \bm{x}}$, the optimal payment that corresponds to $b^k_{ij} = 1$ is given by
\begin{equation*}
    \rho^{k*}_{ij} = (K-k+1)r - (K-k)g(1)r, \quad i \in \mM, \ k \in \mK.
\end{equation*}
The payment $\rho^k_{ij}$ that corresponds to $b^k_{ij} = 0$ is trivially zero since no robot is assigned to the user.
\end{proposition}

\begin{proof}
See Appendix \ref{app:1}.
\end{proof}

\paragraph{Interpretation of Prop.~\ref{prop:1}}
From Prop.~\ref{prop:1}, the SP charges $r$ for the type $K$ service, and the service price decreases as the type drops. This implies that users with a true service type lower than $K$ will always receive a positive utility from the contract. This is because the payment is used to incentivize the user to request the true service type. Type $K$ users have no choice but to select type $K$ robot. Otherwise, they will receive a negative reward. Therefore, the SP can benefit the most by setting a price equal to $r$ for the type $K$ users. 
However, the SP has to reduce the price to attract the users less than type $K$ to truthfully report their true service types. Otherwise, the users will always request higher type services to benefit more.

Prop.~\ref{prop:1} also implies that the optimal payment $\bm{\rho}^*$ does not distinguish the robots with the same service type. Any type $k$ robot among all $N^k$ robots assigned to a user has the same service price. This is because the payment is used to make users truthfully report their types and avoid resource waste. Then, the SP finds $\tuple{\bm{b}, \bm{x}}$ to maximize allocation efficiency.

\subsection{Allocation Sub-Problem}
The allocation sub-problem aims to find proper positions to deploy different service robots for users, which is
\begin{equation}
\label{eq:contract.alloc}
\begin{split}
    \min_{\bm{b}, \bm{x}} \quad & \sum_{i=1}^M \sum_{k=1}^K \sum_{j=1}^{N^k} p^k_i b^k_{ij} f \left( \norm{q_i - x^k_{j}} \right) \\ 
    \text{s.t.} \quad & \eqref{eq:alloc}.
\end{split}
\end{equation}
However, it is inefficient to solve \eqref{eq:contract.alloc} directly due to the following reasons. First, the nonlinearity of $f$ raises the difficulty of solving the MIP. Second, the sub-problem \eqref{eq:contract.alloc} is in a centralized form and faces scalability issues as the number of users $M$ and robots $N$ increase. It also fails to consider safety criteria like collision avoidance between robots. Furthermore, the sub-problem \eqref{eq:contract.alloc} represents a \emph{robust allocation} which only finds optimal allocation plans in the average sense. It does not take advantage of the result from the payment sub-problem to reduce user uncertainties.
Therefore, we develop a distributed allocation algorithm to complete contract design.

\subsection{Contract with Distributed Allocation}
We note that the SP can leverage the optimal payment rule to make users truthfully report their service types and reduce the problem uncertainty. In other words, the SP can replace $p^k_i$ with either 0 or 1 and count the number $M^k$ of total type $k$ users ($\sum_{k \in \mK} M^k = M$). Then, the allocation sub-problem of each type becomes independent. We write the allocation sub-problem for type $k$ as 
\begin{equation}
\label{eq:contract.alloc.k}
\begin{split}
    \min_{\bm{b}^k, \bm{x}^k} \quad & L^k(\bm{b}^k, \bm{x}^k) := \sum_{i=1}^{M^k} \sum_{j=1}^{N^k} b^k_{ij} f\left( \norm{q_i - x^k_{j}} \right) \\ 
    \text{s.t.} \quad & \sum_{j=1}^{N^k} b^k_{ij} = 1, \quad \forall i \in \mM.
\end{split}
\end{equation}

To alleviate the problem stability and incorporate safety considerations, we develop a distributed allocation algorithm to solve \eqref{eq:contract.alloc.k} for each $k \in \mK$.
We consider the single integrator dynamical model for every robot, i.e., $x^k_{j,t+1} = x^k_{j,t} + u^k_{j,t}$, where the additional subscript $t$ denotes the time and $u^k_{j,t}$ is the control input for the robot $j$ with type $k$. 
At time $t$, all $N^k$ type $k$ robots communicate with the SP about their current locations $\bm{x}^k_{t}$ and receive the assignment $b^k_{ij,t}$ and the corresponding user location from the SP. Then, each robot uses the control to move to the new position. We use the negative gradient of the locational energy to design control for each robot, i.e., $u \sim -\nabla_x \sum_{i \in \mM} b^k_{ij} f\left( \norm{q_i - x} \right)$, which can be independently computed by robots. 

\begin{remark}
In the communication with the SP, it is sufficient for each robot to know the assigned users to compute control. If $b^k_{ij} = 0$ for some user $i$, the SP need not tell the robot $j$ anything about the user $i$. Here, for compactness, we write $\sum_{i \in \mM} b^k_{ij} f\left( \norm{q_i - x} \right)$ for each robot.
\end{remark}

We set a safety region with a radius $r_{\text{safe}}$ to perform collision avoidance. The barrier function 
\begin{equation*}
    \Phi^k_j(x) = \begin{cases}
        \sum_{l \in \mB^k_j} -\beta \log \left( \frac{\norm{x - x_l}}{r_{\text{safe}}}  \right) & \norm{x - x_l} < r_{\text{safe}} \\ 
        0 & \norm{x - x_l} \geq r_{\text{safe}}
    \end{cases}
\end{equation*}
is used to measure the safety condition, where $\beta > 0$ is the weighting parameter and $\mB^k_j$ is the set of the neighboring robots in the safety region of the type $k$ robot $j$. Note that $\mB^k_j$ can contain robots other than type $k$.

Finally, we summarize the incentive-compatible contract design with distributed allocation in Alg.~\ref{alg:1}.

\begin{algorithm}
\KwInput{User positions $\bm{q}$, true type $\bm{\theta}$, robots initial positions $\bm{x}^k_0$ for all $k$;}
SP designs the payment rule $\bm{\rho}^*$ by solving \eqref{eq:contract.pay} and sends it to users \;
Users report service type $\bm{\phi}$ based on $\bm{\rho}^*$ \;
SP determines the user's type\;
\For{each type $k = 1,\dots, K$ in parallel}{
    $t \gets 0$ \;
    Locational energy $L^k_t \gets \infty$ \;
    \While{true}{
        All robots report $\bm{x}^k_t$ to the SP \;
        SP computes locational energy $L^k_t$ \;
        SP computes the allocation plan $\bm{b}^k_t$ and sends it to robots \;
        \For{each robot $j=1,\dots, N^k$ in parallel}{
            Detect neighbor set $\mB_j$ in the safety region \;
            $u_1 = -\alpha \nabla_x \sum_{i \in \mM} b^k_{ij,t} f(\norm{q_i - x^k_{j,t}})$ \;
            $u_2 = -\nabla_x \Phi^k_j(x)\vert_{x = x^k_{j,t}}$ \;
            $u^k_{j,t} = \begin{cases}
                u_1 + u_2 & \norm{u_1+u_2} \leq 1 \\
                \frac{u_1+u_2}{\norm{u_1 + u_2}} & \norm{u_1+u_2} > 1
            \end{cases}$ \;
            $x^k_{j,t+1} \gets x^k_{j,t} + u^k_{j,t}, \ \forall j \in \mN^k$ \;
        }
        \uIf{$\abs{L^k_t - L^k_{t-1}} < \epsilon$ \KwOr $t>t_{\max}$}{
            \KwBreak
        }
        Append trajectories $x^k_{j,t}$, $j \in \mN^k$\;
        $t \gets t+1$ \;
    }
}
\caption{Incentive-compatible Contract with Distributed Allocation}
\label{alg:1}
\end{algorithm}

The distributed allocation reduces the computational burden of the SP in implementing the contract. Compared with \eqref{eq:contract.alloc}, the SP only decides user assignment based on the reported robot positions, which is the $\arg\min$ operation over a $M^k$-dimensional vector. Besides, the user assignment for different types can also be executed in parallel. 
The choice of control provides a gradient descent flow in the locational energy $L^k$. Thus, robots will move to the positions and serve the users that correspond to a local optimal solution of \eqref{eq:contract.alloc.k}. The following proposition ensures the convergence of Alg.~\ref{alg:1}.

\begin{proposition} \label{prop:2}
Let the robot control be the gradient $u_1$ in Alg.~\ref{alg:1}. Then, Alg.~\ref{alg:1} converges to some local optimal solution of \eqref{eq:contract.alloc.k}.
\end{proposition}

\begin{proof}
See Appendix \ref{app:2}.
\end{proof}

When all robots reach the target destination, the contract is completed. Then, the robots begin to provide service to the user. If multiple users are assigned to one robot, the robot can solve a routing problem ( e.g., TSP) to visit the user by turns, which is beyond the scope of this work.

\section{Simulations and Evaluations} \label{sec:experiment}
\subsection{Simulation Settings} \label{sec:experiment.setting}
We choose a $10\times 10$ working space for simulation and experiment with eight scenarios with different numbers of users and robots shown in Tab.~\ref{tab:exp_spec}. We randomly generate user and robot information in each scenario, including their positions and type distributions. The user's service type is sampled from the corresponding type distributions.
We use $r=10$ to represent the user's service gain and set $g(x) = \frac{1}{2\log(K)} \log(x+1) + 1$ as the user's extra gain function. The locational energy function is chosen by $f(x) = x^2$. The gradient step $\alpha$ and the weighting parameter $\beta$ are set by $0.1$ and $10$, respectively.

\begin{table}[]
    \centering
    \begin{tabular}{c | >{\centering\arraybackslash}p{1cm} | >{\centering\arraybackslash}p{1cm} | >{\centering\arraybackslash}p{1cm} |c|c} \hline
        Scenario & $M$ & $N$ & $K$ & User type \# & Robot type \# \\ \hline
        \multirow{2}{*}{1-2} &\multirow{2}{*}{20} & 9 & \multirow{2}{*}{3} & \multirow{2}{*}{(5,11,4)} & (3,4,2) \\ 
         &  & 12 &  & & (4,6,2) \\ \hline
        \multirow{2}{*}{3-4} &\multirow{2}{*}{30} & 12 & 3 & (12,11,7) & (4,6,2) \\ 
         & & 15 & 4 & (7,9,10,4) & (3,5,4,3) \\ \hline 
        \multirow{2}{*}{5-6} &\multirow{2}{*}{50} & 15 & \multirow{2}{*}{4} & \multirow{2}{*}{(15,14,11,10)} & (3,5,4,3) \\ 
         & & 40 & & & (12,8,13,7) \\ \hline
        \multirow{2}{*}{7-8} &\multirow{2}{*}{100} & 40 & 4 & (30,25,28,17) & (12,8,13,7) \\ 
         & & 40 & 5 & (18,22,33,20,7) & (8,10,10,9,3) \\ \hline
    \end{tabular}
    \vspace{2mm}
    \caption{Specifications for eight simulation scenarios.}
    \label{tab:exp_spec}
\end{table}

\subsection{Simulation Results of Contract-based Allocation} \label{sec:experiment.contract}

To validate our approach, we run $50$ different simulations with randomly generated user/robot information for every scenario. To make the performance comparable across different scenarios, we use the same initial settings for the users and the robots that have the same type numbers. For example, in Scenario 1-2, the users have the type numbers $(5,11,4)$, and we generate the same user positions and type distributions in these two scenarios. Similarly, in Scenario 5-6, users share the same initial settings; in Scenario 2-3, Scenario 4-5, and Scenario 6-7, robots share the same initial settings. Note that the user/robot information differs in $50$ simulation cases in one scenario; the simulation cases with the same index across the two scenarios have the same settings. Then, we summarize the results with mean and standard deviation in Tab.~\ref{tab:result}.

As we observe, the optimal payment $\bm{\rho}^*$ does not differentiate users and remains the same once the user types are fixed. The SP always charges $r$ to the users who need the highest service type to maximize the profit. The service price decreases for the users with a lower service type. Besides, the convergence of our algorithm can be directly observed by the step column. 

When fixing the users, including their total numbers and types, we see that more service robots lead to a more efficient allocation (faster convergence and smaller locational energy). For example, Scenario 2 has a smaller mean step and mean locational energy than Scenario 1; the same applies to Scenario 5 and 6. This is because more robots are used for service assignment, and a robot needs to serve fewer users on average. Therefore, robots can select the allocation positions to reduce the distance between the served users, leading to more efficient allocation plans.

When fixing the robots, we observe that more users result in a less efficient allocation with longer convergent steps and higher locational energy. The comparisons between Scenarios 2 and 3, Scenarios 4 and 5, and Scenarios 6 and 7 illustrate this point. This is because a robot has more users to serve on average and needs more routing time to visit all robots. Besides, increasing the number of users also leads to more possible user positions, some of which may either facilitate or worsen the robot allocation that starts from the same initial condition. It results in a larger deviation in the locational energy, as we can observe in the comparisons.

\begin{table}[]
    \centering
    \begin{tabular}{c|c|>{\centering\arraybackslash}p{2.5cm}|>{\centering\arraybackslash}p{2.5cm}} \hline
        Scenario & $\bm{\rho}^*$ (follow user type \#) & Step & Energy \\ \hline
        \multirow{2}{*}{1-2} & \multirow{2}{*}{(5.0,7.5,10.0)} & 55.82(14.16) & 72.22(23.76) \\ 
         & & 48.48(12.55) & 53.53(19.41) \\ \hline
        \multirow{2}{*}{3-4} & (5.0,7.5,10.0) & 54.58(15.83) & 108.06(26.15) \\ 
         & (3.54,5.69,7.85,10.0) & 53.74(12.33) & 89.28(18.67) \\ \hline 
        \multirow{2}{*}{5-6} & \multirow{2}{*}{(3.54,5.69,7.85,10.0)} & 55.42(11.76) & 182.24(34.20) \\ 
         & & 45.24(12.26) & 55.22(13.26) \\ \hline
        \multirow{2}{*}{7-8} & (3.54,5.69,7.85,10.0) & 49.80(12.18) & 131.54(19.87) \\ 
         & (2.26,4.20,6.13,8.07,10.0) & 49.74(12.12) & 166.00(28.94) \\ \hline
    \end{tabular}
    \vspace{2mm}
    \caption{Simulation results of contract-based allocation. The mean and variance are obtained per 50 simulations. $\rho^*$ denotes the optimal payment rule. The total steps for allocation and the final locational energy after the allocation are presented accordingly.}
    \label{tab:result}
\end{table}

To visualize distributed allocation, we show the robot trajectories of a specific allocation simulation in Fig.~\ref{fig:alloc_trajectory} using the setting in Scenario 8. We plot the trajectories by service types for clarity due to the large number of users. After using the payment rule to identify the user types, all robots successfully go to the allocation positions to serve the assigned users. Some zigzag trajectories, e.g., trajectories in Fig.~\ref{fig:type.4}, are due to collision avoidance. The dashed line indicates which users are assigned to the corresponding robots after the distributed allocation. Depending on their random initial positions, some robots may serve one user (e.g., Fig.~\ref{fig:type.4}) while some robots serve multiple users (e.g., Fig.~\ref{fig:type.5}). No collision is detected when all robots move during the allocation.

\begin{figure*}
    \captionsetup[subfigure]{justification=centering}
    \begin{subfigure}[b]{0.19\textwidth}
        \centering
        \includegraphics[height=2.5cm]{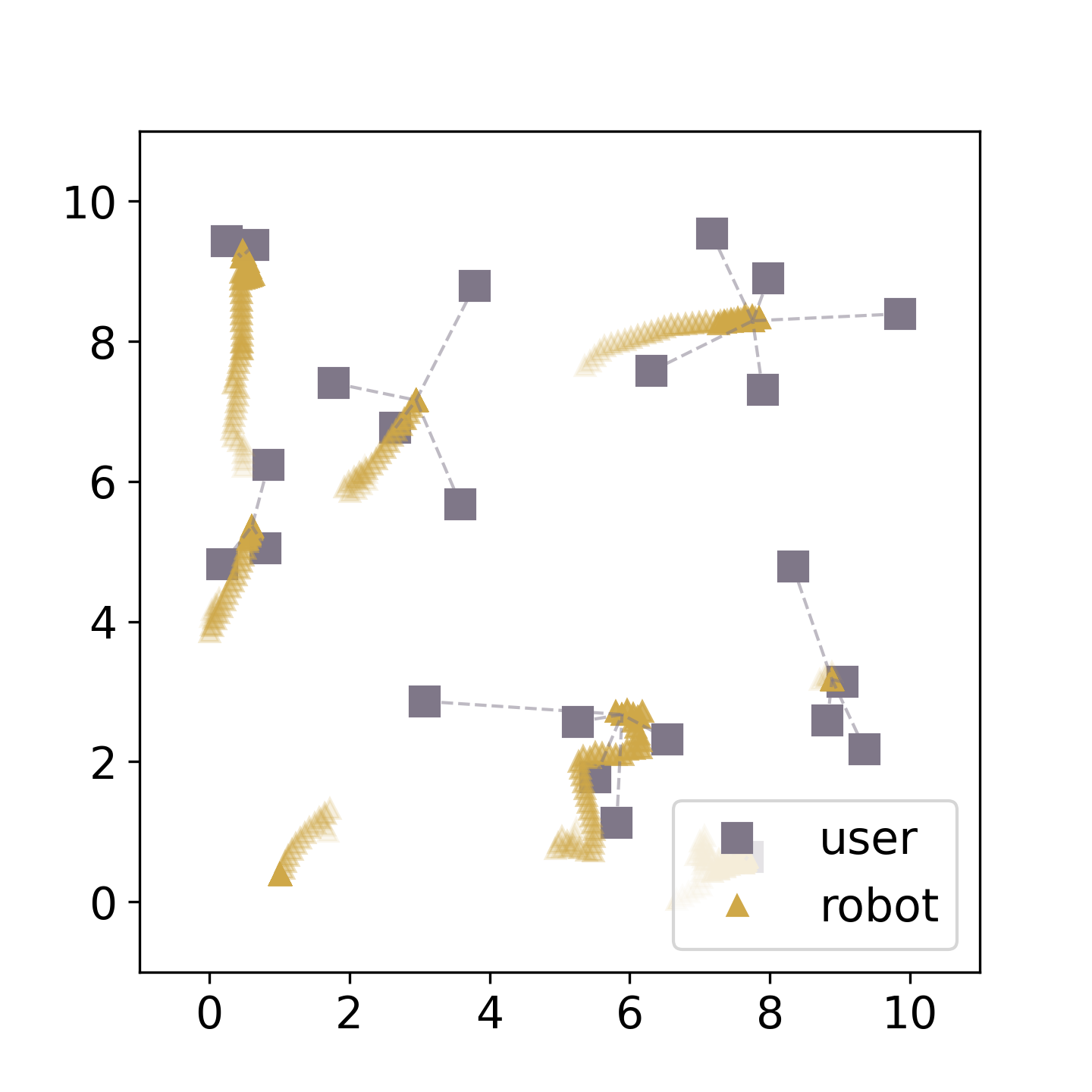}
        \caption{Type 1 allocation.}
        \label{fig:type.1}
    \end{subfigure}
    \captionsetup[subfigure]{justification=centering}
    \begin{subfigure}[b]{0.19\textwidth}
        \centering
        \includegraphics[height=2.5cm]{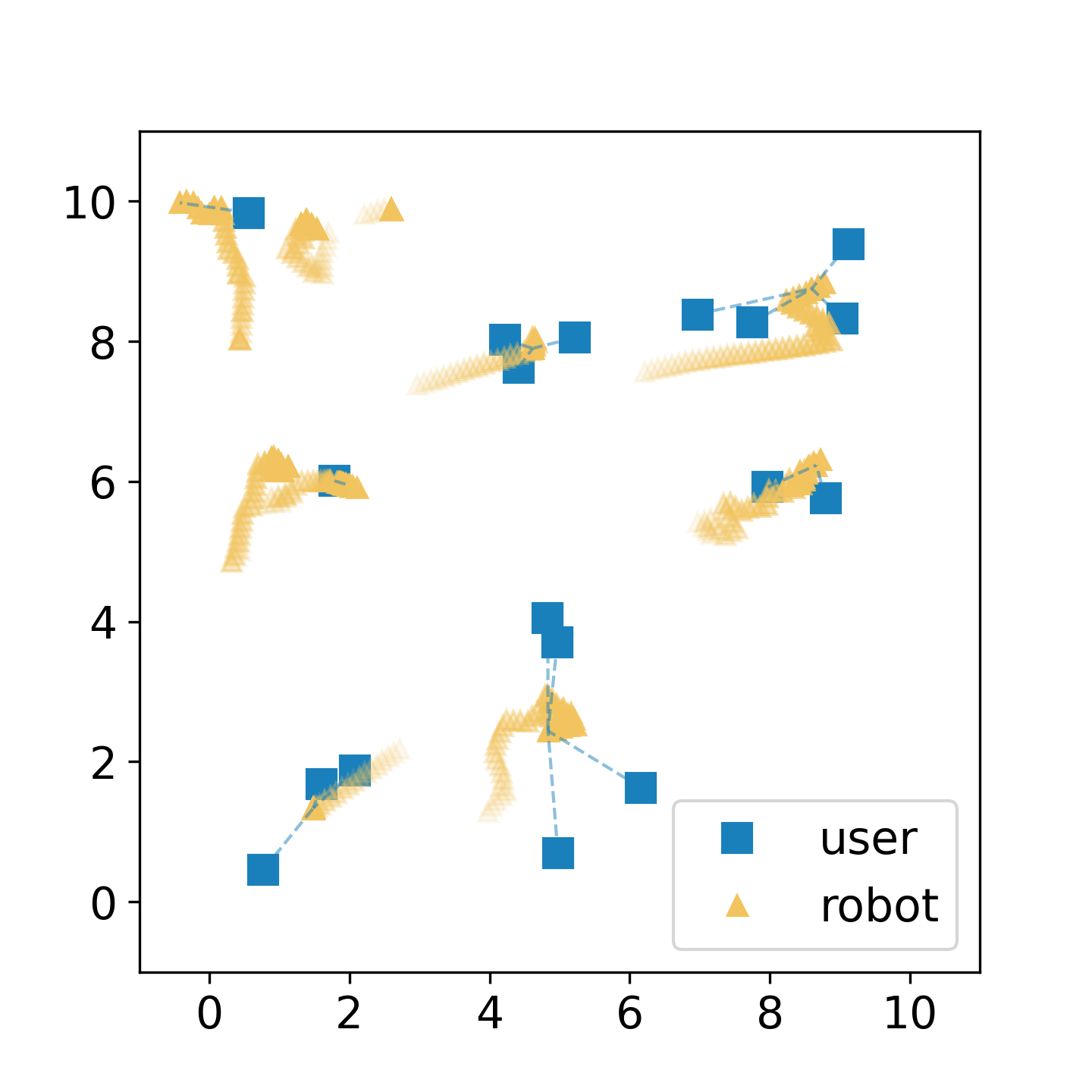}
        \caption{Type 2 allocation.} 
        \label{fig:type.2}
    \end{subfigure}
    \captionsetup[subfigure]{justification=centering}
    \begin{subfigure}[b]{0.19\textwidth}
        \centering
        \includegraphics[height=2.5cm]{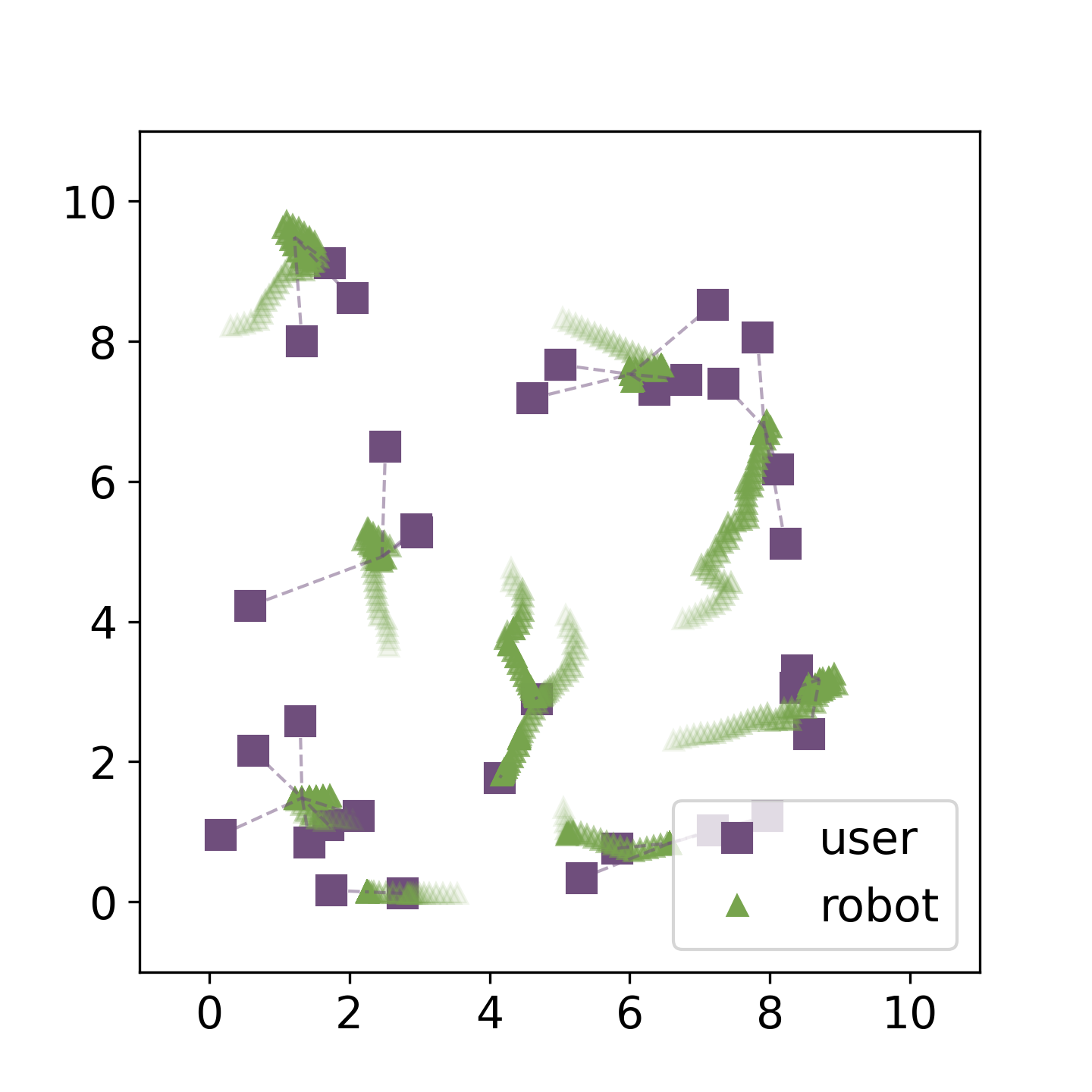}
        \caption{Type 3 allocation.} 
        \label{fig:type.3}
    \end{subfigure}
    \captionsetup[subfigure]{justification=centering}
    \begin{subfigure}[b]{0.19\textwidth}
        \centering
        \includegraphics[height=2.5cm]{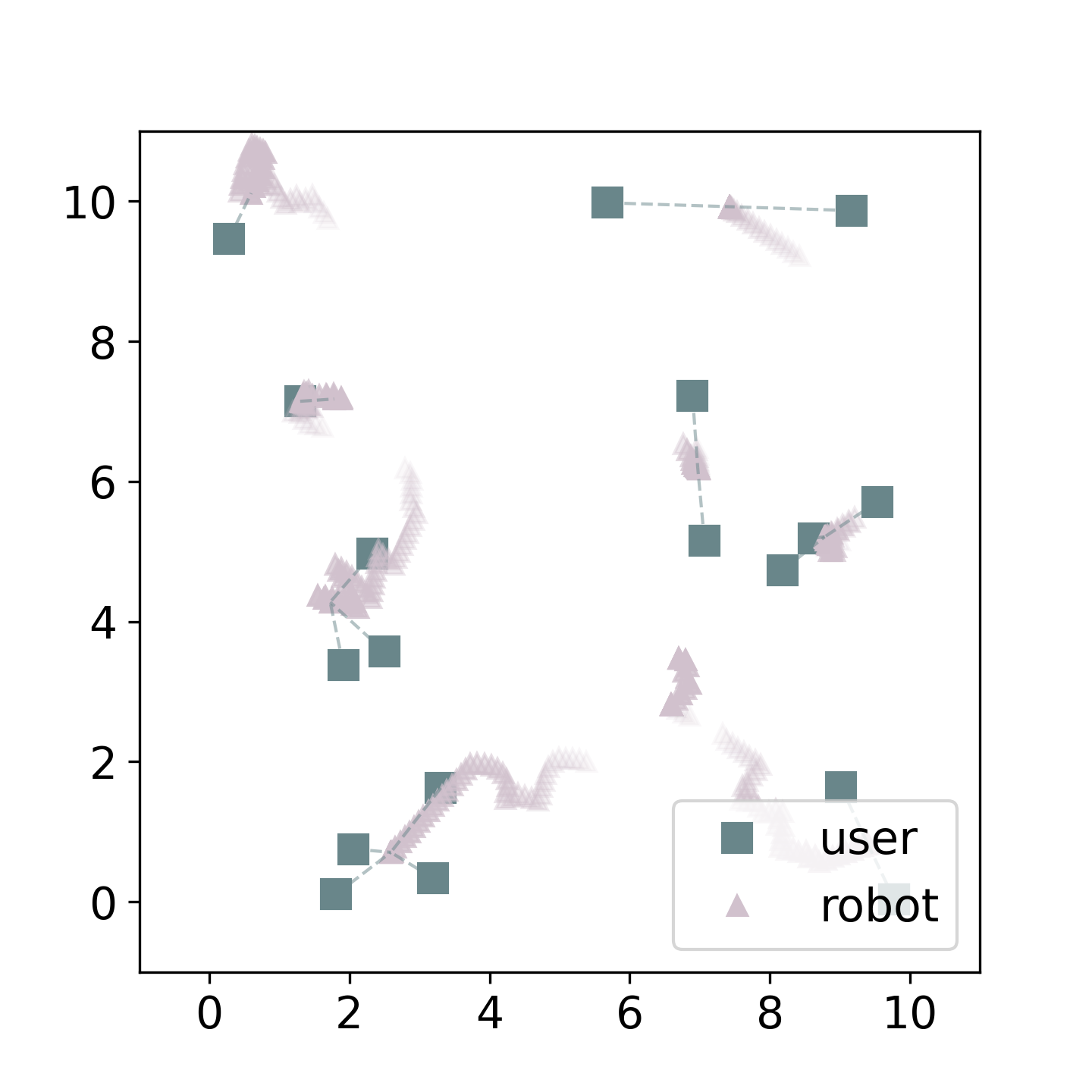}
        \caption{Type 4 allocation.} 
        \label{fig:type.4}
    \end{subfigure}
    \captionsetup[subfigure]{justification=centering}
    \begin{subfigure}[b]{0.19\textwidth}
        \centering
        \includegraphics[height=2.5cm]{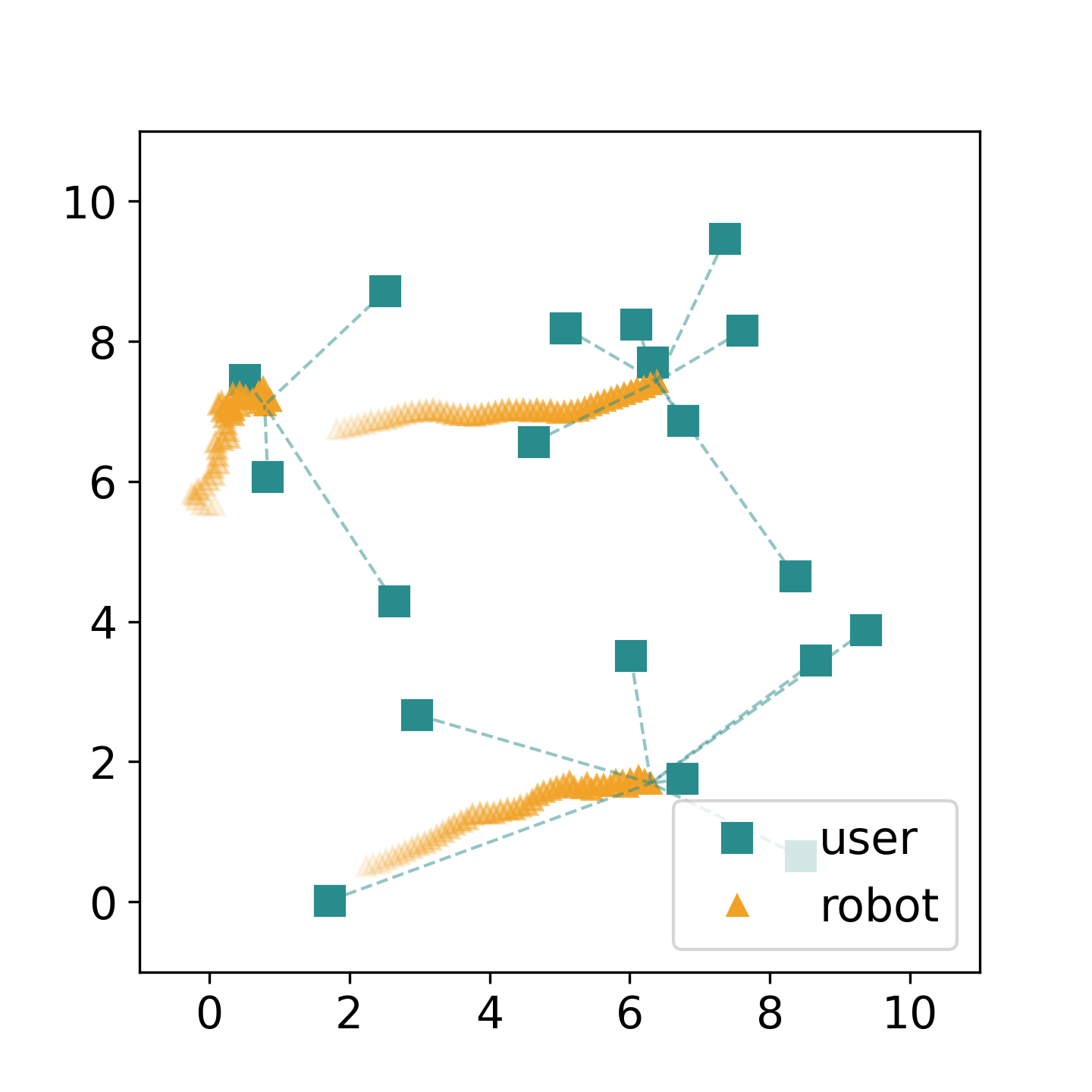}
        \caption{Type 5 allocation.} 
        \label{fig:type.5}
    \end{subfigure}
    \caption{The allocation trajectories of the robot with different service types in Scenario 8 (five service types and 100 users). After identifying the user type with the optimal payment, all service robots with different types successfully go to users using our distributed allocation algorithm. Some zigzag trajectories are mainly due to collision avoidance. The dashed lines indicate the user assignment after the allocation.}
    \label{fig:alloc_trajectory}
\end{figure*}

\subsection{Comparison with Other Uncertainty Reduction} \label{sec:experiment.compare}

Our contract-based approach effectively reduces the uncertainty of the user's true service type and hence achieves higher allocation efficiency compared with the robust allocation \eqref{eq:contract.alloc}. Here, we compare our method with two non-contract-based uncertainty reduction methods: \emph{random-max matching} and \emph{random-sample matching}. 
Given the user's type distribution $p_i \in \Delta(\mK), i \in \mM$, random-max matching determines the user $i$'s type by choosing the one with the largest probability mass; random-sample matching samples the user $i$'s type in $\mK$ based on the distribution $p_i$. Both uncertainty reduction methods reduce the type distribution to deterministic types. Then, the SP uses these deterministic types to perform robot allocation.
We use the same simulation settings as in Sec.~\ref{sec:experiment.contract} to test these methods alongside the robust allocation and summarize the terminal locational energies in Tab~\ref{tab:comparison}. Different locational energies are obtained with the user types generated by the corresponding method.

\begin{table}[]
    \centering
    \begin{tabular}{c|>{\centering\arraybackslash}p{2.5cm}|>{\centering\arraybackslash}p{2.5cm}|>{\centering\arraybackslash}p{2.5cm}} \hline
         Scenario & robust & max & samp  \\ \hline
         1 & 119.35(20.75) & 94.21(28.33) & 88.54(30.12) \\ 
         2 & 99.98(21.16) & 75.06(24.11) & 66.52(26.44) \\ \hline
         3 & 160.53(27.21) & 125.95(39.68) & 134.61(32.88) \\ 
         4 & 142.44(17.87) & 104.20(28.06) & 100.25(26.33) \\ \hline
         5 & 236.59(24.92) & 200.31(38.45) & 192.99(28.76) \\ 
         6 & 86.67(9.24) & 60.04(13.76) & 57.67(12.71) \\ \hline
         7 & 179.52(13.87) & 140.41(20.37) & 137.25(20.71) \\ 
         8 & 266.99(18.34) & 208.31(36.55) & 204.81(32.28) \\ \hline
    \end{tabular}
    \vspace{2mm}
    \caption{Final locational energies of robust allocation (robust), random-max matching (max), and random-sample matching (samp). The mean and deviation are obtained per 50 simulations using the same settings as Tab.~\ref{tab:result}.}
    \label{tab:comparison}
\end{table}

The robust allocation shows the largest locational energy since it considers average cases for allocation. The two comparing methods show lower locational energies due to uncertainty reduction. 
However, despite the higher efficiency, the uncertainty reduction approach can cause incorrect estimations of the user's true type and assign mismatched robots to users. A mismatch occurs when a lower-type robot is assigned to serve a user with a higher service-type request. Note that a higher-type robot can fulfill the lower-type user's demand despite resource waste. We summarize the mismatches of the two methods in Tab.~\ref{tab:mismatch}. 
As the number of users increases, more mismatches occur, resulting in service provision failures. The SP sacrifices service provision accuracy although she thinks she has improved allocation efficiency.
We also observe that, on average, choosing the user type with the maximum probability mass appears to be a safer method to eliminate uncertainty than randomly sampling the type due to smaller mismatches and variations. 

Compared with our approach in Tab.~\ref{tab:result}, the two comparing methods show higher locational energies on average. This is because the SP may assign a robot to serve more users based on incorrect user type information. The robot then needs more routing time to visit all assigned users, leading to higher locational energy and less efficient allocation. 

We note that the comparison between mean-variance values of Tab.~\ref{tab:result} and Tab.~\ref{tab:comparison} cannot reflect case-by-case performance, making the comparison less straightforward. Therefore, we compute the locational energy difference in each simulation and show the mean-variance per 50 simulations for all scenarios in Fig.~\ref{fig:err}. The energy difference is obtained by subtracting our result from the locational energy generated by other methods.
It is expected that robust allocation always produces the largest difference. On average, our method produces more efficient allocations than the other two uncertainty reduction methods shown by the positive mean value. This indicates that the accuracy of uncertainty reduction also has an impact on allocation efficiency. We note some negative deviations (e.g., in Scenario 5), which means that the comparing methods have lower locational energy in some simulation cases. However, this only implies that the robots, in these cases, happen to provide an efficient allocation corresponding to the reduced user type in the price of mismatches.

\begin{table}[]
    \centering
    \begin{tabular}{c|c|c|c|c} \hline
        Scenario & 1 & 2 & 3 & 4  \\ \hline
        max & 5.02(1.61) & 5.02(1.61) & 5.44(1.86) & 7.52(2.29) \\ 
        samp & 5.50(1.92) & 5.52(1.68) & 6.90(1.75) & 9.12(2.44) \\ \hline 
        Scenario & 5 & 6 & 7 & 8 \\ \hline
        max & 13.12(2.92) & 13.12(2.92) & 23.98(3.53) & 27.14(3.45) \\ 
        samp & 14.48(2.90) & 14.60(2.93) & 28.00(3.76) & 31.04(4.26) \\ \hline
    \end{tabular}
    \vspace{2mm}
    \caption{Mismatches in user assignment for different scenarios. A mismatch occurs when a low-type robot is assigned to a user with a high-type service request.}
    \label{tab:mismatch}
\end{table}

\begin{figure}
    \captionsetup[subfigure]{justification=centering}
    \begin{subfigure}[b]{0.48\textwidth}
        \centering
        \includegraphics[height=4cm]{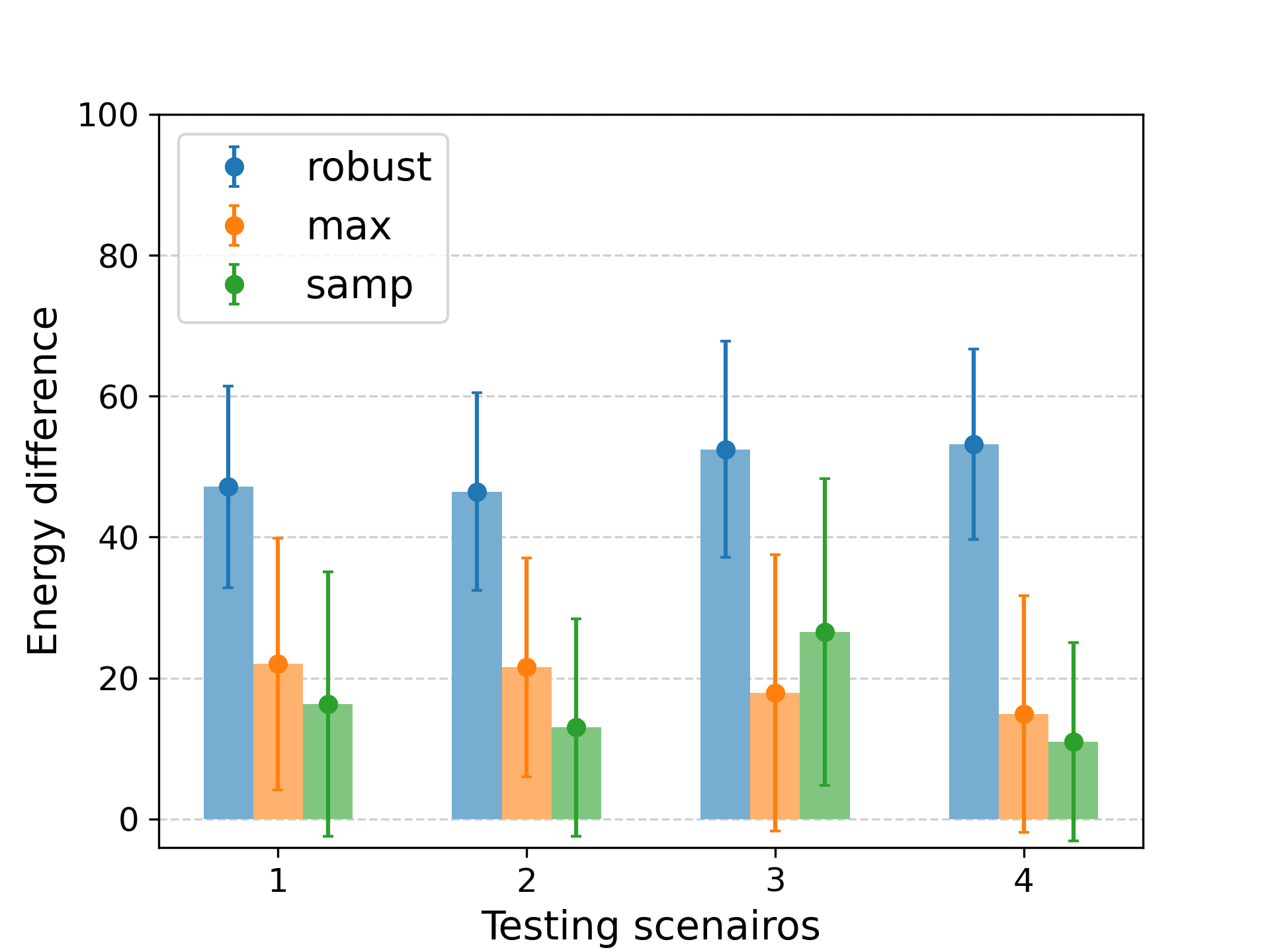}
        \caption{Scenario 1-4.} 
        \label{fig:err.s1_4}
    \end{subfigure}
    \captionsetup[subfigure]{justification=centering}
    \begin{subfigure}[b]{0.48\textwidth}
        \centering
        \includegraphics[height=4cm]{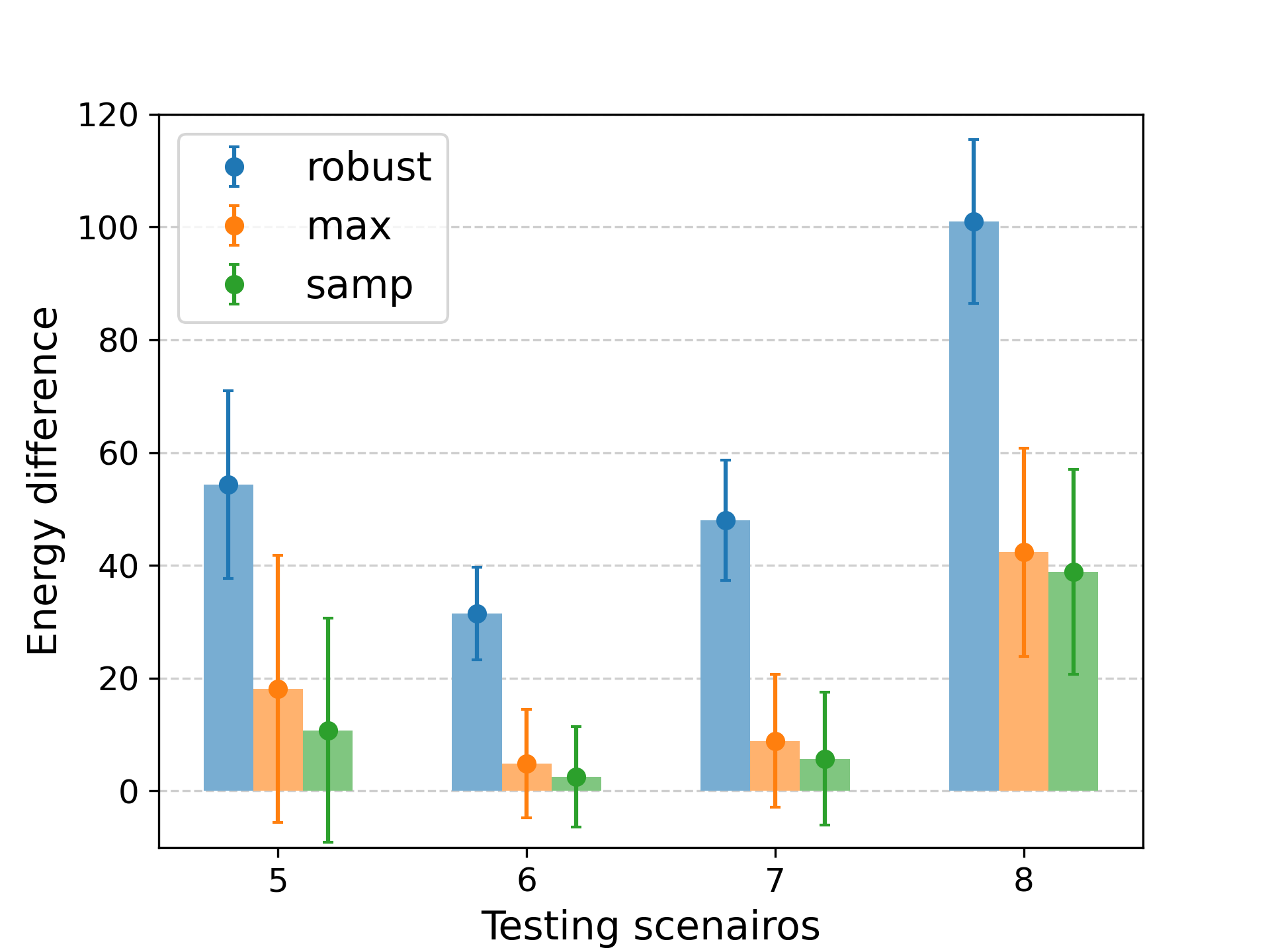}
        \caption{Scenario 5-8.} 
        \label{fig:err.s5_8}
    \end{subfigure}
    \caption{Case-by-case locational energy differences per 50 simulations for different scenarios. The energy difference is obtained by subtracting our result from comparing methods.}
    \label{fig:err}
\end{figure}

\section{Conclusions} \label{sec:conclusion}
In this work, we have introduced a contract-based approach to allocating service robots to different users under uncertainties regarding user service requirements. The incentive-compatible design principle ensures that the user truthfully selects the real service demands among different service options and optimizes resource utilization. The developed distributed allocation algorithm effectively guides all robots to the target location for service provision by following collision-free trajectories. Case studies have corroborated the effectiveness of our approach and demonstrated that our method achieves better allocation efficiency and accuracy compared with the robust allocation paradigm and two uncertainty reduction methods. 

As we observe, analytical contract-based solutions require explicit user models. In practice, we may need extra efforts to model users, including using data-driven methods, which is one of our future directions. We would also extend our approach to a dynamic allocation framework for the sequential arrival of user demands in uncertain environments.

\appendix
\section{Proof of Proposition \ref{prop:1}} \label{app:1}

Given any allocation variable $\tuple{\bm{b}, \bm{x}}$, the allocation constraints \eqref{eq:alloc} indicate that each type $k$ user is only assigned to one robot. Therefore, the constraints \eqref{eq:ir}-\eqref{eq:ic2} are only meaningful for nonempty allocation with $b^k_{ij} = 1$. For each user $i$, WLOG, we omit the subscript $ij$ and write $\rho^k, k \in \mK$, as the payment variable corresponding to the nonempty allocation $b^k_{ij} = 1$.
Then, we simplify \eqref{eq:ir}-\eqref{eq:payment} by
\begin{equation} 
\label{eq:ir.pf.1}
    0 \leq \rho^k \leq r, \quad k \in \mK,
\end{equation}
\begin{equation}
\label{eq:ic1.pf}
    \rho^k - \rho^l \leq r, \ \forall l = 1,\dots k-1, k \in \mK \backslash \{1\},
\end{equation}
\begin{equation}
\label{eq:ic2.pf}
    \rho^l - \rho^k \geq g(l-k)r - r, \ \forall l = k+1,\dots, K, k \in \mK \backslash \{K\}.
\end{equation}
Note that the constraints \eqref{eq:ir.pf.1} always imply the constraint \eqref{eq:ic1.pf}. Thus, we only need to focus on \eqref{eq:ir.pf.1} and \eqref{eq:ic2.pf}.

We observe that the objective in \eqref{eq:payment} given $\bm{b}$ is linear in $\bm{\rho}$ and each coefficient $p^k_i$ is positive. It means that each $\rho^k, k \in \mK$, needs to be maximized.
Since $\rho^K$ has the least couplings in \eqref{eq:ic2.pf}, we set $\rho^{K*} = r$ to achieve the maximum payment for the type $K$ service. Then, we simplify \eqref{eq:ic2.pf} with $\rho^{K*}$ and obtain
\begin{equation}
\label{eq:ir.pf.2}
    \rho^k \leq 2r - g(K-k) r, \quad k \in \mK \backslash \{K\}.
\end{equation}
It is clear that $2r - g(K-k) r < r$. Therefore, \eqref{eq:ir.pf.2} provides a new upper bound for \eqref{eq:ir.pf.1}. Next, we set $\rho^{K-1*} = 2r - g(1)r$ as the upper bound to maximize the payment for the type $K-1$ service. We keep updating \eqref{eq:ic2.pf} with $\rho^{K-1*}$ and obtain
\begin{equation}
\label{eq:ir.pf.3}
    \rho^{k} \leq 3r - g(1) r - g(K-k-1) r, \quad k \in \mK \backslash \{K,K-1\}.
\end{equation}
Now we have two upper bounds \eqref{eq:ir.pf.1} and \eqref{eq:ir.pf.2} on $\rho^k, k \in \mK \backslash \{K,K-1\}$. Using the concavity of $g$, we can show that $3r - g(1)r - g(K-k-1) r < 2r- g(K-k)r $ for all $ k \in \mK \backslash \{K,K-1\}$. Therefore, \eqref{eq:ir.pf.3} provides a new upper bound for $\rho^k$ and we can set $\rho^{K-2*} = 3r - 2g(1)r$ and ignore \eqref{eq:ir.pf.2}. 

We update \eqref{eq:ic2.pf} with $\rho^{K-2*}$ and repeat the analysis. By induction, we can obtain
\begin{equation*}
    \rho^{k*} = (K-k+1)r - (K-k)g(1)r, \quad k \in \mK.
\end{equation*}
Since $g(1) < \frac{K}{K-1}$, we can verify that $\rho^{k*} \geq 0, k \in \mK$, and thus feasible. 
The analysis holds for arbitrary $ij$ pair once the allocation variable is specified. 

For the empty allocation $b^k_{ij} = 0$, the constraints \eqref{eq:ir}-\eqref{eq:ic2} trivially holds and the corresponding $\rho^k_{ij}$ becomes free variables. we set these $\rho_{ij} = 0$ since there is no assignment and it is meaningless to have any positive values.

\section{Proof of Proposition \ref{prop:2}} \label{app:2}

For simplicity, we ignore the superscript type $k$ since the proof applies to all types of users and robots. When fixing the robot position $\bm{x}$, the optimal user assignment is given by the Voronoi partition $V = \{V_1,\dots, V_{N}\}$ based on $f$, where
\begin{equation*}
    V_j = \{ q_i, i \in \mM : f\left( \norm{q_i - x_j} \right) \leq f\left( \norm{q_i - x_l}\right), \forall l \in \mN \}.
\end{equation*}
The optimal assignment variable can be derived by $b^*_{ij} = 1$ if $q_i \in V_j$ and $0$ otherwise, $\forall i \in \mM$.
The locational energy under $\bm{b}^*$ can be written as
\begin{equation*}
    L_b(\bm{x}) := L(\bm{x}, \bm{b}^*(\bm{x})) = \sum_{i=1}^M \min_{j = 1,\dots, N} f\left( \norm{q_i - x_j} \right).    
\end{equation*}
For any assignment $\bm{b}$ satisfying the allocation constraint \eqref{eq:alloc}, it is clear that $L_b(\bm{x}) \leq L(\bm{x}, \bm{b})$. From \cite{du1999centroidal}, we have
\begin{equation*}
    \frac{\partial L_b(\bm{x})}{\partial x_j} = \sum_{q_l \in V_j} \frac{\partial f \left( \norm{q_l - x_j} \right)}{\partial x_j}, \quad \forall j \in \mN.
\end{equation*}
After receiving $\bm{b}^*$ from the SP, the robot control is in fact the gradient of $L_b(\bm{x})$, i.e., $u \sim - \frac{\partial L_b}{\partial x_j}$.
Thus, we have 
\begin{equation*}
    x_{j,t+1} = x_{j,t} - \alpha \frac{\partial L_b}{\partial x_j} \bigg\vert_{x=x_{j,t}}, \quad \forall j \in \mN.
\end{equation*}
With a proper step size, we have $L(\bm{x}_{t+1}, \bm{b}^*(\bm{x}_t)) \leq L_b(\bm{x}_t)$. Also, since $\bm{b}^*(\bm{x}_t)$ is not the optimal assignment given $\bm{x}_{t+1}$, we have $L_b( \bm{x}) \leq L(\bm{x}_{t+1}, \bm{b}^*(\bm{x}_t))$, which implies $L_b(\bm{x}_{t+1}) \leq L_b (\bm{x}_t)$. The algorithm will converge to the centroids of some Voronoi partitions $V^*$ where $x^*_j = \frac{1}{\abs{V_j^*}} \sum_{q_l \in V^*_j} q_l$, which yields a zero gradient. The resulting $\bm{x}^*$ and $V^*$ correspond to a local minimizer of $L(\bm{x}, \bm{b})$.

%
\bibliographystyle{splncs04}
\bibliography{mybib}
\end{document}